\documentclass[journal,final,compsoc]{IEEEtran}
\usepackage[dvipsnames]{xcolor}

\usepackage[nocompress]{cite}
\usepackage{amsthm}
\usepackage{amssymb}
\usepackage{amsmath}
\usepackage{psfrag}
\usepackage[]{graphicx}
\usepackage{overpic}

\newtheorem{theorem}{Theorem}
\newtheorem{lemma}[theorem]{Lemma}

\newtheorem{corollary}[theorem]{Corollary}

\graphicspath{{matlab/}}

\newcommand{\cb}[1]{{\ifmmode {\boldsymbol{#1}}\else ${\boldsymbol{#1}}$\fi}}
\newcommand{\cp}[1]{\ifmmode {\mathcal{#1}}\else ${\mathcal{#1}}$\fi}

\newcommand{\bx}{\cb{x}}
\newcommand{\bxc}{{\cb{x}_\mathrm{\!c}}}
\newcommand{\bX}{\cb{X}}
\newcommand{\bM}{\cb{M}}
\newcommand{\bP}[1]{\cb{P}\!\!_{_#1}}
\newcommand{\bmu}{\cb{\mu}} 
\newcommand{\bw}{\cb{w}}
\newcommand{\bwc}{{\cb{w}_\mathrm{\!c}}}
\newcommand{\lambdac}{{\lambda_\mathrm{c}}}
\newcommand{\bW}{\cb{W}}

\newcommand{\balpha}{\cb{\alpha}}
\newcommand{\balphac}{{\cb{\alpha}_\mathrm{\!c}}}

\newcommand{\bbeta}{\cb{\beta}}

\newcommand{\bA}{\cb{A}}
\newcommand{\bAc}{{\cb{A}_\mathrm{\!c}}}

\newcommand{\bC}{\cb{C}}
\newcommand{\bCc}{{\cb{C}_\mathrm{\;\!\!\!c}}}
\newcommand{\bCn}{{\cb{C}}}

\newcommand{\bLambda}{\cb{ \Lambda}} 
\newcommand{\bLambdac}{{\cb{ \Lambda}_\mathrm{\!c}}} 
\newcommand{\pic}{{\pi_\mathrm{\!c}}}

\newcommand{\bomega}{\cb{\omega}}

\newcommand{\dc}{{d_\mathrm{\!c}}}
\newcommand{\bv}{\cb{v}}
\newcommand{\bdelta}{\cb{\delta}}
\newcommand{\bDelta}{\cb{\Delta}}
\newcommand{\bDeltas}{{\cb{\bDelta}_{\xi}}}

\newcommand{\bK}{{\cb{K}}}
\newcommand{\bKc}{\cb{K}_\mathrm{\!c}}
\newcommand{\bXc}{\cb{X}_\mathrm{\!c}}
\newcommand{\bI}{{\bf I}}

\newcommand{\tr}{\mathrm{tr}}

\newcommand{\bkappa}{\cb{\kappa}}
\newcommand{\bone}{\cb{1}}
\newcommand{\bzero}{\cb{0}}

\newcommand{\red}{\color{red}}
\newcommand{\blue}{\color{blue}}

\begin{document}


\title{An eigenanalysis of data centering \\in machine learning}

\author{Paul Honeine,~\IEEEmembership{Member,~IEEE}
\IEEEcompsocitemizethanks{\IEEEcompsocthanksitem M.
Honeine is with the Institut Charles Delaunay (CNRS), Universit\'{e} de Technologie de Troyes, Troyes, France.
}
}%

\IEEEcompsoctitleabstractindextext{
\begin{abstract}

Many pattern recognition methods rely on statistical information from centered data, with the eigenanalysis of an empirical central moment, such as the covariance matrix in principal component analysis (PCA), as well as partial least squares regression, canonical-correlation analysis and Fisher discriminant analysis. Recently, many researchers advocate working on non-centered data. This is the case for instance with the singular value decomposition approach, with the (kernel) entropy component analysis, with the information-theoretic learning framework, and even with nonnegative matrix factorization. Moreover, one can also consider a non-centered PCA by using the second-order non-central moment. 

The main purpose of this paper is to bridge the gap between these two viewpoints in designing machine learning methods. To provide a study at the cornerstone of kernel-based machines, we conduct an eigenanalysis of the inner product matrices from centered and non-centered data. We derive several results connecting their eigenvalues and their eigenvectors. Furthermore, we explore the outer product matrices, by providing several results connecting the largest eigenvectors of the covariance matrix and its non-centered counterpart. These results lay the groundwork to several extensions beyond conventional centering, with the weighted mean shift, the rank-one update, and the multidimensional scaling. Experiments conducted on simulated and real data illustrate the relevance of this work.
\end{abstract}


}

\maketitle

\section{Introduction}

Most pattern recognition in machine learning can be explained by performing an eigenanalysis, with an eigendecomposition or a spectral decomposition ({\em i.e.}, singular value decomposition or SVD) \cite{Sun2009}. These machines seek a set of relevant axes from a given dataset. The principal component analysis (PCA) \cite{PCA} is the most prominent eigenanalysis problem for feature extraction and dimensionality reduction. In this case, the most relevant axes, obtained from the eigendecomposition of the covariance matrix, capture the largest amount of variance in the data. Other machines include multidimensional scaling \cite{MDS}, partial least squares regression (PLS) \cite{PLSR}, canonical-correlation analysis (CCA) \cite{CCA} and its classification-based version known as Fisher discriminant analysis (FDA)\cite{Fis36}. The latter two methods solve a generalized eigendecomposition problem.

Kernel-based machines provide an elegant framework to generalize these linear pattern recognition methods to the nonlinear domain. They rely on the concept of \emph{kernel trick}, initially introduced by Aizerman \emph{et al.} in \cite{Aiz64}. The main breakthrough lies in two folds. On the one hand, most pattern recognition, classification and regression algorithms can be written in terms of inner products between data. On the other hand, by substituting each inner product by a (positive definite) kernel function, a nonlinear transformation is implicitly operated on the data without any significant computational cost. Therefore, the eigenanalysis as well as most of the operations are performed on the kernel matrix, which corresponds to an inner product ``Gram'' matrix in some feature space. This property is revealed in kernelized versions of PCA \cite{Sch99}, PLS \cite{Ros02}, CCA \cite{Fukumizu2007} and FDA \cite{MikRaeWesSchMue99}. See \cite{Shawetaylor_Cristianini,Bishop06,Has09} for a survey of kernel-based machines.

In several kernel-based machines, as given for instance in PCA, CCA, and FDA, data should be centered in the feature space, by shifting the origin to the centroid of the data. From an algorithmic point of view, centering is performed easily with matrix algebra, either in batch mode by a subsequent column and row centering of the kernel matrix \cite{Sch99}, or in a recursive way when dealing with online learning \cite{12.tpami}. From a theoretical point of view, centering reveals central moments, {\em i.e.,} moments about the centroid/mean of the available data, as well as other related statistics. Well-known central moments include the second-order central moment, also called covariance, which is investigated by the PCA for estimating the maximum-variance directions. Furthermore, these directions minimize the reconstruction error.

Many researchers advocate the use of non-centered data in pattern recognition. Information is extracted directly, either with the data matrix and its Gram matrix, or with non-central moments. Several motivations were revealed in favor of working on non-centered data. The intuitive motivation is the application of the spectral decomposition without data-centering in many pattern recognition and machine learning problems, thus providing a sort of a non-centered PCA by using the second-order non-central moment. This is the case for instance in signal analysis and classification \cite{svd_feature} and in designing dictionaries for sparse representation \cite{KSVD}. A key motivation towards keeping data non-centered is the nonparametric density estimation with kernel functions \cite{Gir02}, as revealed recently with exceptional performance in the (kernel) entropy component analysis (ECA) \cite{KECA} and the information-theoretic learning framework \cite{Principe2010}. See also \cite{Jenssen2009}. A further motivation is that data often deviate from the origin, and the measure of such deviation may constitute an interesting feature, such as in hyperspectral unmixing \cite{unmix_overview}. It turns out that in many fields of computer science, signal processing and machine learning, acquired data are nonnegative, and even positive. This is the case with the study of gene expressions in bioinformatics \cite{Yeung01092001}, with eigenfaces in the computer vision problem of human face recognition \cite{eigenfaces}, with online handwritten character recognition \cite{pca_ocr}, and with numerous applications for nonnegative matrix factorization \cite{lee99,ICAhandbook10}. As a consequence, nonnegative Gram matrices are pervasive. Such information is lost when centering the data, as well as several interesting properties\footnote{For instance, the Perron--Frobenius theorem \cite{berman1994nonnegative,Hor12} states that, under some mild conditions, the non-negativity of a matrix is inherited by its unique largest eigenvalue and that the corresponding eigenvector has positive components. This result is no longer valid when the data are centered.}.





The issue of centering the data versus keeping the data uncentering is an open question in pattern recognition: (Pearson) correlation versus congruence coefficients, (centered) PCA versus non-centered PCA (or SVD), covariance and centered Gram matrices versus their non-centered counterparts. In this paper, we study the impact of centering the data on the distribution of the eigenvalues and eigenvectors of both inner-product and outer-product matrices. By examining the Gram matrix and its centered counterpart, we show the interlacing property of their eigenvalues. We devise bounds connecting the eigenvalues of these two matrices, including a lower bound on the largest eigenvalue of the centered Gram matrix. Furthermore, we examine the eigenvectors of the inner product and outer product matrices. We provide connections between the most relevant eigenvector of the covariance matrix and that of the non-centered matrix, a result that corroborate the work in \cite{PCAuncenter}.

In our study, we focus on the eigenanalysis of the Gram matrices. This work opens the way to understanding the impact of centering the data in most kernel-based machine. This is shown by bridging the gap between the (centered) PCA and the (non-centered) ECA. Moreover, our work goes beyond the PCA and ECA, since it extends naturally to many kernel-based machines where the eigen-decomposition of the Gram matrix is crucial. To this end, we revisit the multidimensional scaling problem where centering is essential. Moreover, we provide extension beyond conventional mean-centering.


\subsection*{Related (and unrelated) work}

In machine learning for classification and discrimination, the issue of centering the data has been addressed in few publications. In the Bayesian framework proposed in \cite{Gestel2002}, the bias-free formulation of support vector machines (SVM) and least-squares SVM yields the eigendecomposition of the centered Gram matrix, while Gaussian processes yield similar expressions applied on the non-centered Gram matrix. In \cite{Park2005}, the authors propose a modified FDA to take into account the fact that centering leads to a singular Gram matrix, even when the non-centered Gram matrix is non-singular. More recently, it is devised in \cite{Cortes05} that data and label should be centered when dealing with the alignment criterion. In \cite{Markovsky09}, the author consider the issue of optimizing the centering as well as the low-rank approximation problem.

In Bayesian statistics, the impact of centering has been extensively studied in multilevel models, when dealing with hierarchically nested models \cite{Raudenbush89,Longford89}. In this case, centering is performed either with the grand mean, or ``partially'' with a group mean centering, at different levels \cite{Hofmann1998,Paccagnella06}. See \cite[Chapter 5.2]{kreft1998} for a comprehensive review on the centering issue in multilevel modeling. This problem is revisited within a Bayesian framework in \cite{Papaspiliopoulos03,Yu2011}, with the issue of centered or non-centered parameterisations. This is beyond the scope of this paper, since we investigate non-parametric methods such as PCA and ECA.

 To the best of our knowledge, only Cadima and Jolliffe investigated in \cite{PCAuncenter} the relationship between the PCA and its non-centered variant. To this end, they confronted the eigendecomposition of the covariance matrix with the eigendecomposition of its non-centered counterpart. Such connection can be done thanks to the rank-one update that connects both matrices. In this paper, we study the eigendecompositions of the centered and the non-centered Gram matrices. It is easy to see that this is a much harder problem. By performing an eigenanalysis of the Gram matrices, we provide a framework that integrates the analysis of PCA, ECA, MDS, and beyond. This work opens the door to the study of centering or not in kernel-based machines. 

\subsection*{Notation}

The matrix $\bI$ is the $n$-by-$n$ identity matrix, and the vector $\bone$ is the all-ones vector of $n$ entries. Then, $\bone\bone\!^\top$ is the $n$-by-$n$ matrix of all ones, and $\bone\!^\top\!\bone=n$. Moreover, $\bone\!^\top\! \bM \bone$ is the grand sum of the matrix $\bM$. The subscript $~_\mathrm{\!c}$ allows to recognize the case of centered data, as opposed to the 
non-centered one: $\bXc$ versus $\bX$, $\bKc$ versus $\bK$, $\bCc$ versus $\bCn$, $\lambdac_{i}$ versus $\lambda_{i}$, etc. 

The spectral decomposition ({\em i.e.}, singular value decomposition) of a matrix $\bM$ defines its singular values $\sigma_1(\bM), \sigma_2(\bM), \ldots$, given in non-increasing order. The $i$-th largest eigenvalue of a square matrix $\bM$ is denoted by $\lambda_{i}(\bM)$, and its trace by $\tr(\bM)$. A first analysis on the eigenvalues of a matrix is given by its trace, with:
\begin{equation}\label{eq:trace}
	\tr(\bM) = \sum_i \lambda_i(\bM).
\end{equation}
This corresponds to the variance of the data when dealing with the data covariance matrix.

For the sake of clarity, the $i$-th largest eigenvalues of the inner product matrices $\bK$ and $\bKc$ are respectively $\lambda_{i}$ and $\lambdac_{i}$, namely $\lambda_{i} = \lambda_{i}(\bK)$ and $\lambdac_{i} = \lambdac_{i}(\bKc)$. 
The eigenpair $(\lambda_{i},\balpha_i)$ of $\bK$ denotes its $i$-th largest eigenvalue $\lambda_{i}$ and its corresponding eigenvector $\balpha_i$. The {\em first} eigenvector is the one associated to the largest eigenvalue.

\section{Introduction}

In this section, we present the eigendecomposition problems of the inner product and outer product matrices, for both non-centered and centered data. Then, two kernel-based machines are presented, PCA and ECA, in order to contrast the paradigm of centering or not the data.

\subsection{Non-centered data}

Consider a set of $n$ available samples, $\bx_1,\bx_2, \ldots, \bx_n$, from a vector space of dimension $d$, with the conventional inner product $\bx_i^\top \bx_j$. Let $\bX= [\bx_1 ~~ \bx_2 ~ \cdots ~ \bx_n]$ be the data matrix, and $\bK = \bX\!^\top\! \bX$ be the corresponding Gram matrix, {\em i.e.,} the inner product matrix. It turns out that this matrix encapsulates the essence of the information in the data, as illustrated with the kernel trick throughout the literature. 
It is therefore natural to focus on the Gram matrix in out study.

Let $(\lambda_i,\balpha_i)$ be an eigenpair of the matrix $\bK$, namely
\begin{equation}\label{eq:eig_K}
	\bK \, \balpha_i = \lambda_i \, \balpha_i,
\end{equation}
for $i=1,2,\ldots, n$, the eigenvalues $\lambda_1, \lambda_2, \ldots, \lambda_n$ being given in non-increasing order. These quantities describe the spectral decomposition of the matrix $\bK$, with
\begin{equation}\label{eq:svd_K}
	\bK = \bA \, \bLambda \, \bA\!^\top.
\end{equation}
The eigendecomposition of the Gram matrix $\bK$ is related to the spectral decomposition of the data matrix $\bX$, since the latter is given by
\begin{equation}\label{eq:svd_X}
	\bX = \bW \bLambda^{\frac12} \bA\!^\top,
\end{equation}
where $\bA$ is the $n$-by-$n$ matrix whose columns are the eigenvectors $\balpha_1, \balpha_2, \ldots, \balpha_n$, and $\bLambda^{\frac12}$ is the $d$-by-$n$ rectangular diagonal matrix whose $i$-th diagonal entry is 
\begin{equation}\label{eq:sigma}
	\sigma_i(\bX) = \sqrt{\lambda_i}.
\end{equation}

The $d$ columns of $\bW$ are called the left-singular vectors of $\bX$. They define the spectral decomposition of the matrix $\bX\bX\!^\top$, which is known as the realized covariation matrix \cite{Barndorff-Nielsen2004} in financial economics. For a coherent analysis, we consider in this paper the second-order {\em non-central} moment matrix $\bCn = \tfrac1n \sum_{i=1}^n \bx_i \, \bx_i^\top$, written in matrix form as $\bCn = \tfrac1n \bX\bX\!^\top$. Following the spectral decomposition of $\bX$ in \eqref{eq:svd_X}, we get
\begin{equation*}
	\bC \, \bw_i = \tfrac{1}{n} \lambda_i \, \bw_i,
\end{equation*}
where $\bW = [\bw_1 ~~ \bw_2 ~~ \cdots ~~ \bw_d]$. The matrix $\bCn$ is less known than the second-order {\em central} moment matrix. The latter, called covariance matrix, is obtained after centering the data, as given in the following.

\subsection{Centered data}

Consider centering the data, $\bxc_i = \bx_i - \bmu$ for $i=1,2, \ldots, n$, where $ \bmu = \frac1n \sum_{i=1}^n \bx_i$ is the empirical mean ({\em i.e.}, centroid). We get in matrix form $\bXc = \bX - \bmu \bone\!^\top$, where
\begin{equation*}
	\bmu = \tfrac1n \bX \bone.
\end{equation*}
From the spectral decomposition of the matrix $\bK$ in \eqref{eq:svd_K}, we can rewrite the norm of the mean as
\begin{equation}\label{eq:norm_mean}
	\|\bmu\|^2 
	= \tfrac{1}{n^2} \bone\!^\top\! \bK \bone
	= \tfrac{1}{n^2} \sum_{i=1}^n \lambda_i (\balpha_i\!^\top\! \bone)^2,
\end{equation}
which illustrates how each pair of eigenvalue and eigenvector contributes to the norm of the data mean. The impact of centering the data is clear on both the Gram and the covariance matrices, as illustrated next.

Let $\bKc=\bXc\!\!^\top\! \bXc$ be the Gram matrix of the centered data, with entries $\bxc_i^\top \bxc_j$ for $i,j=1,2, \ldots, n$, namely
\begin{align}\label{eq:Kc_Kn0}
	{\bKc} 
    &= \bK - \tfrac1n \bone\bone\!^\top\! \bK 
		- \tfrac1n \bK \bone\bone\!^\top  
		+ \tfrac{1}{n^2} \bone\bone\!^\top\! \bK \bone\bone\!^\top,
\end{align}
Let $(\lambdac_i,\balphac_i)$ be an eigenpair of this matrix, then:
\begin{equation}\label{eq:eig_Kc}
	\bKc \, \balphac_i = \lambdac_i \, \balphac_i.
\end{equation}

Let $\bCc=\tfrac1n \bXc \bXc\!^\top$ be the covariance matrix, namely the second-order central moment matrix defined by 
\begin{align}\label{eq:Cc_Cn0}
\bCc 	
	= \bCn - \bmu \bmu\!^\top.
\end{align}
The eigenpairs of this matrix define the eigenproblem
\begin{equation}\label{eq:eig_Cc}
	\bCc \, \bwc_i = \tfrac{1}{n} \lambdac_i \, \bwc_i.
\end{equation}

\subsection*{Nonlinear extension using kernel functions}


A symmetric kernel $\kappa(\cdot,\cdot)$ is called a positive definite kernel if it gives rise to a positive definite matrix, namely for all $n\in I\!\!N$ and $\bx_1,\bx_2, \ldots, \bx_n$ we have  
\begin{equation}\label{eq:pd_kernel}
	\bbeta\!^\top\! \bK \bbeta \geq 0,
\end{equation}
for all vectors $\bbeta$, where the matrix $\bK$ has entries $\kappa(\bx_i,\bx_j)$. Such kernel functions provide a nonlinear extension of the conventional inner product since, thanks to Mercer's theorem \cite{Mercer,reproducingkernel}, $\kappa(\bx_i,\bx_j)$ corresponds to an inner product between transformed samples $\bx_i^\phi$ and $\bx_j^\phi$ in some feature space, namely
\begin{equation}\label{eq:kernel}
	\kappa(\bx_i,\bx_j) = {\bx_i^\phi}^\top \bx_j^\phi.
\end{equation}
Examples of nonlinear kernels include the polynomial kernel, of the form $(c+\bx_i^\top \bx_j)^p$, and the Gaussian kernel $\exp(-\tfrac{1}{2\sigma^2} \|\bx_i - \bx_j\|^2)$ where $\sigma$ is the tunable bandwidth parameter. 

It turns out that the derivations given in this paper can be easily generalized to kernel functions, since a kernel matrix corresponds to a Gram matrix in some feature space. Centering is performed in the feature space, with the centroid\footnote{In most pattern recognition tasks, one does not need to quantify $\bmu^\phi$, but only its inner product with any $\bx_i^\phi$. One can estimate its counterpart in the input space. This is the pre-image problem, which is clearly beyond the scope of this paper. See \cite{11.spm} for a recent survey.} $\bmu^\phi = \tfrac1n \bX^\phi \bone$. As a consequence, we get the same expression as in \eqref{eq:norm_mean}, with $\| \bmu^\phi \|^2 = \tfrac{1}{n^2} \bone\!^\top\! \bK \bone$.



%
%
%
%

\subsection{The paradigm of centering or not the data}

In order to confront centering the data with keeping the data non-centered, we present next two well-known machines for pattern recognition. On the one hand, advocating the use of central matrices, the PCA is presented in its two-folds, the conventional and the kernelized formulations. On the other hand, advocating the exploration of non-centered data, nonparametric density estimation using the ECA. 


\subsubsection{Case study 1: Principal component analysis}\label{sec:PCA}

The PCA seeks the axes that capture most of the variance within the data \cite{PCA}. It is well-known\footnote{By writing these vectors in a matrix $\bW$, maximizing the variance of the projected data can be written as $\arg\max \tr(\bW \bCc \bW\!^\top)$, under the constraint $\bW\!^\top\! \bW = \bI$. By using the Lagrangian multipliers and taking the derivative of the resulting cost function, we get the corresponding eigenproblem.} that these axes are defined by the eigenvectors associated to the largest eigenvalues of the covariance matrix $\bCc$, with 
\begin{equation*}
	\bCc \, \bwc_i = \tfrac{1}{n} \lambdac_i \, \bwc_i.
\end{equation*}
Then, $\tfrac{1}{n} \lambdac_i$ measures the variance along the axe $\bwc_i$, while the normalized $i$-th eigenvalue $\pic_i = \frac{\lambdac_i}{\sum_j \lambdac_j}$ accounts for the proportion of total variation. It is worth noting from \eqref{eq:trace} that the total variance of the data is given by the trace of $\bCc$.

By substituting the definition of $\bCc$ in the eigenproblem \eqref{eq:eig_Cc}, we get
\begin{align}\label{eq:wci_eigen}
	\bwc_i = \tfrac{n}{\lambdac_i} \bCc \, \bwc_i 
		= \tfrac{1}{\lambdac_i} \sum_{j=1}^n (\bxc_j^\top \bwc_i) \, \bxc_i, 
\end{align}
and therefore each $\bwc_i$ lies in the span of the data. This means that there exists a vector $\balphac_i$ such that $\bwc_i = \bXc \, \balphac_i$. By injecting this relation in the above expression, we get $\bXc \, \balphac_i = \tfrac{1}{\lambdac_i} \bXc \bXc\!\!^\top\! \bXc \, \balphac_i$. Equivalently, by multiplying each side by $\bXc\!\!^\top$, we get $\bKc \, \balphac_i = \tfrac{1}{\lambdac_i} \bKc^2 \balphac_i$. Thus, we have the following eigenproblem
\begin{equation*}
	\bKc \, \balphac_i = {\lambdac_i} \, \balphac_i.
\end{equation*}
The eigenvectors of $\bKc$ allow to identify the projection of any sample $\bx$ onto the corresponding eigenvectors of $\bCc$. Representing it on a set of eigenvectors is given by
\begin{equation}\label{eq:project}
	\sum_k (\bx\!^\top\! \bwc_k) \, \bwc_k
	= \sum_k  (\bx\!^\top\! \bXc \balphac_k) \bXc \balphac_k.
\end{equation}

In order to satisfy the unit-norm of $\bwc_i$ for any $i$, the eigenvectors of $\bKc$ should be normalized. Indeed, we have 
$\bwc_i^\top \bwc_i 
	= \balphac_i^\top \bXc\!\!^\top\! \bXc \balphac_i
	= \balphac_i^\top \bKc \balphac_i
	= \lambdac_i \balphac_i^\top \balphac_i$. 
Therefore, one can define each feature $\bwc_i$ directly from the eigenvector $\balphac_i$ of the Gram matrix $\bKc$, after normalization such that 
\begin{equation}\label{eq:kpca_normalization}
	\|\balphac_i\|^2 = \frac{1}{\lambdac_i}.
\end{equation}
This scaling allows to preserve the variance of the data along the respective axes \cite{KPCA}. When one needs to fix the scale along all these directions, the following normalization is considered:
\begin{equation}\label{eq:kpca_normalization2}
	\|\balphac_i\|^2 = \frac{1}{\lambdac_i^2},
\end{equation}
This scaling allows to set a unit variance within projected data on each axe \cite{Tax02}.

\subsubsection{Case study 2: Nonparametric density estimation}\label{sec:KECA}

Nonparametric density estimation is essential in many applied mathematical problems. Many machine learning techniques are based on density estimation, often with a Parzen window approach\cite{Parzen1962,Izenman1991}. This is the case of the information-theoretic methods \cite{Principe2010}, which are essentially based on the quadratic R\'enyi entropy, of the form $ - \log \int p(\bx)^2 d\bx$ for a probability density $p(\cdot)$. It is also the case of the (kernel) entropy component analysis (ECA) for data transformation and dimensionality reduction \cite{KECA}. See also \cite{Gir02} for more details.

Often unknown, the probability density $p(\cdot)$ is estimated using a Parzen estimator of the form $ \hat{p}(\bx)=\frac{1}{n}\sum_{i=1}^n \kappa(\bx,\bx_i)$ for a given kernel function centered at each available sample $\bx_i$. By using the kernel matrix $\bK$, the estimator is given by
\begin{equation*}
	\hat{p}(\bx) = \frac{1}{n}\sum_{i=1}^n \kappa(\bx,\bx_i)
			= \tfrac1n \bone\!^\top \bkappa(\bx),
\end{equation*}
where $\bkappa(\bx)$ is the vector of entries $\kappa(\bx,\bx_i)$ for $i=1,2, \ldots, n$. From the definition \eqref{eq:kernel}, it is easy to see that this expression corresponds to the inner product between the sample and the mean, with
\begin{equation*}
	\hat{p}(\bx) = {\bx^\phi}^\top \Big(\tfrac{1}{n}\sum_{i=1}^n \bx_i^\phi \Big)
			= {\bx^\phi}^\top \bmu^\phi
\end{equation*}
The quantity $\int p(\bx)^2 d\bx$, related to the quadratic R\'enyi entropy, is therefore estimated by $\int \hat{p}(\bx)^2 d\bx = \| \bmu^\phi \|^2$, which leads to
\begin{equation}\label{eq:keca}
	\int \hat{p}(\bx)^2 d\bx 
	=\tfrac{1}{n^2} \sum_{i=1}^n \lambda_i (\balpha_i\!^\top\! \bone)^2,
\end{equation}
where \eqref{eq:norm_mean} is used. This expression uncovers the composition of the entropy in terms of eigenvectors of $\bK$. This is the main motivation of the ECA, where the relevant eigenvectors are selected in order to maximize the estimated entropy, thus the smallest terms $\lambda_i (\balpha_i\!^\top\! \bone)^2$. Therefore, any eigenvector $\balpha_i$ for which $\balpha_i\!^\top\! \bone\neq0$ and $\lambda_i\neq0$ contributes to the entropy estimate, in contrast with the PCA where only eigenvectors associated to non-zero eigenvalues contribute to the variance.


Finally, we emphasize that non-centered data is used in the density estimation. Centering the data leads to a null density estimator, which yields an infinite quadratic R\'enyi entropy. This fact is also corroborated by several studies, including the (kernel) entropy component analysis. See \cite{KECA} for more details.

\subsection{Features from centered vs non-centered data}

Even in conventional PCA, the issue of centering is still an open question\footnote{For instance in R (The R Project for Statistical Computing) \cite{R}, there are two ways to performs PCA: On the one hand, the {\em R-mode} by using the eigendecomposition of the covariance matrix as given in \eqref{eq:eig_Cc}, with the function {\tt princomp}; and on the other hand, the {\em Q-mode} by using the singular value decomposition of the non-centered data, as given in \eqref{eq:svd_X}, with the function {\tt svd}.}. To the best of our knowledge, only the work in \cite{PCAuncenter} studied the link between the eigendecompositions of $\bCn$ and $\bCc$. Expression \eqref{eq:Cc_Cn0} reveals the rank-one update nature between these two matrices. The analysis of the Gram matrices $\bK$ and $\bKc$ is obviously a much harder problem, as illustrated in expression \eqref{eq:Kc_Kn0}.

In this paper, we take the initiative to study the of the inner product matrices, $\bK$ and $\bKc$, and carry on with the eigenanalysis of the outer product matrices, $\bCn$ and $\bCc$. It turns out that the work \cite{PCAuncenter} on $\bCn$ and $\bCc$ can be derived easily from the proposed approach. Moreover, the study of the inner product matrices broadens the scope of the work to the analysis of all kernel-based techniques \cite{KSVD,KECA,Shawetaylor_Cristianini}, beyond the PCA approach. Next section gives the main contributions of this paper, while this study is completed in Section~\ref{sec:extensions} by several extensions, beyond centering.


\section{Main results}

The main results are given next, in two-folds: the (inner product) Gram matrix and the (outer product) covariance matrix. For each, we describe the relations between the eigenvectors of the centered matrix with those obtained from the non-centered one. The relations between the eigenvalues are studied by examining the inner product matrices, which allows the generalization of these results to nonlinear kernel functions. But before, we need to briefly introduce the orthogonal projection, and its link to the centering operation.

\subsubsection*{Background on orthogonal projections}

The matrix of orthogonal projections onto the subspace spanned by the columns of a given matrix $\bM$ is defined by $\bP{\bM} = \bM \big( \bM\!^\top \bM \big)^{-1} \bM\!^\top$, while $\bI - \bP{\bM}$ denotes te projection onto its orthogonal complement. Projections are idempotent transformations, {\em i.e.,} $\bP{\bM}\bP{\bM} = \bP{\bM}$. In particular, we are interested in this paper in the projection onto the all-ones vector $\bone$ of $n$ entries, with
\begin{equation}\label{eq:P1}
	\bP{\bone} = \frac1n \bone \bone\!^\top.
\end{equation}

By considering the projection of the data onto the subspace spanned by $\bone$ and its complement, we have\footnote{Since the data are given column-wise in the matrix $\bX$, the mean $\bmu$ obtained from the operation $\bX \bP{\bone}$ is the vector of means of each row of $\bX$. This is in opposition to the operation $\bP{\bone}\bX$ which provides the means of each of the columns, {\em i.e.,} each $\bx_i$.}
\begin{align}\label{eq:XP}
	{\bX \bP{\bone}} 
		= \bmu \bone\!^\top.
\end{align}
It is easy to see that the matrix of centered data is given by $\bXc = \bX (\bI - \bP{\bone})$ and verifies the identity $\bXc\bone=\bzero$. More generally, we have
\begin{equation}\label{eq:I_P_1}
	(\bI - \bP{\bone}) \bone = \bone\!^\top\! (\bI - \bP{\bone})  = \bzero.
\end{equation}

Finally, for any square matrix $\bM$ of appropriate size:
\begin{equation}\label{eq:PMP}
	\bP{\bone} \bM \bP{\bone} = \tfrac{1}{n^2} \bone \bone\!^\top\! \bM \bone \bone\!^\top = \tfrac{\bone\!^\top\! \bM \bone}{n} \, \bP{\bone}.
\end{equation}
This yields
\begin{equation}\label{eq:trPMP}
	\tr(\bP{\bone} \bM \bP{\bone}) = \tr(\bP{\bone} \bM) = \tr(\bM \bP{\bone}) = 
	\tfrac1n \bone\!^\top\! \bM \bone,
\end{equation}
where the cyclic property of the trace operator is used in the first two equalities. Therefore, we have
\begin{align}\label{eq:trIP_M_IP}
	\tr\big((\bI-\bP{\bone}) \bM (\bI-\bP{\bone}) \big) 
	&= \tr(\bM) - \tfrac1n \bone\!^\top\! \bM \bone.
\end{align}

\subsection{Inner product -- the Gram matrix $\bKc$}\label{sec:K}

Let $\bK=\bX\!^\top\! \bX$ be the Gram matrix, then its centered counterpart is $\bKc = \bXc^\top\! \bXc$, namely
\begin{align}
	{\bKc} &= (\bI - \bP{\bone})\bK(\bI - \bP{\bone}) \label{eq:Kc_Kn} \\
		&= \bK - \bP{\bone}\bK - \bK\bP{\bone} + \tfrac{\bone\!^\top\! \bK \bone}{n} \, \bP{\bone} \label{eq:Kc_Kn2} \\
		&= \bK - \bP{\bone}\bK - \bK\bP{\bone} + n \|\bmu\|^2 \bP{\bone},\label{eq:Kc}
\end{align}
where relation \eqref{eq:PMP} is applied on $\bK$, and the last equality is due to \eqref{eq:norm_mean}. These expressions reveal the {\em double centering}, which corresponds to subtracting the row and column means of the matrix $\bK$ from its entries, and adding the grand mean. 
A byproduct of the double centering is obtained from \eqref{eq:Kc_Kn} and \eqref{eq:I_P_1}: $0$ is the eigenvalue of $\bKc$ associated to the eigenvector $\tfrac1n\bone$.

A first study to understand the distribution of the data associated to each matrix, $\bK$ and $\bKc$, is given by their traces, since they correspond to the sum of all eigenvalues. The following lemma provides a measure of the variance reduction due to centering.
\begin{lemma}\label{th:traceK}
Let $\bK$ and $\bKc$ be respectively the Gram matrix and its centered counterpart, then their corresponding traces verify the following relationships:
\begin{equation*}
	\tr(\bKc) = \tr(\bK) - n \|\bmu\|^2,
\end{equation*}
and 
\begin{equation*}
	\frac{\tr(\bKc)}{\tr(\bK)} = 1 - \frac{\bone\!^\top\! \bK \bone}{n \, \tr(\bK)}.
\end{equation*}
Hence, their eigenvalues verify $\sum_{i=1}^n \! \lambdac_i \!=\! \sum_{i=1}^n \! \lambda_i \!-\! n \|\bmu\|^2$.
\end{lemma}

\begin{proof}
	The proof is straightforward, by applying \eqref{eq:trIP_M_IP} to $\bK$ given in definition \eqref{eq:Kc_Kn}, with $\bone\!^\top\! \bK \bone = n^2 \|\bmu\|^2$.
\end{proof}

Next, we explore beyond the sum of the eigenvalues. In Section~\ref{sec:interlacing}, we show that eigenvalues of $\bK$ are interlaced with those of $\bKc$. In Section~\ref{sec:eigenvalues_bounds}, we provide bounds on a sum of the largest $t$ eigenvalues, for any $t$, and in particular a lower bound on the largest eigenvalue of $\bKc$.

\subsubsection{Eigenvalue interlacing theorems for $\bK$ and $\bKc$}\label{sec:interlacing}

Before proceeding, the following theorem from \cite[Theorem~5.9]{Yanai2011} is central to our study. It is worth noting that this theorem has been known in the literature for some time, see for instance \cite[Appendix~A]{takane1991principal}, and is obtained from the Poincar\'e Separation Theorem \cite{Rao79,Rao80} (see also \cite[Corollary 4.3.37 in page 248]{Hor12}).

\begin{theorem}[Separation Theorem \cite{takane1991principal,Yanai2011}]
\label{th:separation}
Let $\bM$ be a $d$-by-$n$ matrix. Let two orthogonal projection matrices be $\bP{\mathrm{left}}$, of size $d$-by-$d$, and $\bP{\mathrm{right}}$, of size $n$-by-$n$. Then, 
\begin{equation*}
	\sigma_{j+t}(\bM) \leq \sigma_{j}(\bP{\mathrm{left}} ~\bM~ \bP{\mathrm{right}}) \leq \sigma_{j}(\bM),
\end{equation*}
where $\sigma_{j}(\cdot)$ denotes the $j$-th largest singular value of the matrix, $t = d - r(\bP{\mathrm{right}}) +n - r(\bP{\mathrm{left}})$ and $r(\cdot)$ is the rank of the matrix.
\end{theorem}

\begin{theorem}\label{th:ineq_lambda}
Let $\bK$ and $\bKc$ be respectively the Gram matrix and its centered counterpart, then their eigenvalues are interlaced, such that
\begin{equation*}
	\lambda_{j+1} \leq \lambdac_{j} \leq \lambda_{j},
\end{equation*}
with $\lambdac_{n}=0$, where $\lambda_{j}$ and $\lambdac_{j}$ denote respectively the $j$-th largest eigenvalue of the matrices $\bK$ and $\bKc$.
\end{theorem}

\begin{proof}
To prove this, we apply the Separation Theorem~\ref{th:separation} with $\bM = \bX$, $\bP{\mathrm{left}}$ being the $d$-by-$d$ identity matrix, and $\bP{\mathrm{right}} = (\bI - \bP{\bone})$, where $r(\bP{\mathrm{left}})=d$ and $r(\bP{\mathrm{right}})=n-1$, and thus $t=1$. In this case, we get
\begin{equation*}
	\sigma_{j+1}(\bX) \leq \sigma_{j}(\bX(\bI - \bP{\bone})) \leq \sigma_{j}(\bX).
\end{equation*}
Relations \eqref{eq:sigma} and \eqref{eq:Kc_Kn} are used to conclude the proof.
\end{proof}

\begin{corollary}\label{th:ineq_lambda_corollary}
Let $\bK$ and $\bKc$ be respectively the Gram matrix and its centered counterpart, then their proportion of total variation accounted for by the eigenvectors are interlaced, such that
\begin{equation*}
	 \pi_{j+1} \leq \gamma \; \pic_{j} \leq \pi_{j},
\end{equation*}
where $\pi_{j} = \frac{\lambda_j}{\sum_i \lambda_i}$, $\pic_{j} = \frac{\lambdac_j}{\sum_i \lambdac_i}$ and $\gamma = \frac{\tr(\bKc)}{\tr(\bK)}$.
\end{corollary}

\begin{proof}
 The proof is direct, on the one hand by dividing the inequalities of Theorem~\ref{th:ineq_lambda} by the trace of $\bK$, and on the other hand by setting $\gamma = \frac{\tr(\bKc)}{\tr(\bK)}$ with the direct application of \eqref{eq:trace}.
\end{proof}

All these results show the impact of centering the data on the distribution of the eigenvalues, where the eigenvalues of $\bKc$ are {\em sandwitched} between the eigenvalues of $\bK$. This illustrates that $\bKc$ behaves like a ``coarse'' matrix compared to 
$\bK$. 

\subsubsection{Bounds on the eigenvalues of $\bK$ and $\bKc$}\label{sec:eigenvalues_bounds}

In this section, we provide lower bounds on the largest eigenvalues of the matrices $\bK$ and $\bKc$. To this end, we state the Schur--Horn Theorem~\cite{Horn54}. See \cite[Chapter 9]{majorization2011} for a recent review on the theory of majorization.

\begin{theorem}[Schur--Horn Theorem]\label{th:Schur_Horn}
	For any $n$-by-$n$ symmetric matrix with diagonal entries $d_1,d_2, \ldots, d_n$ and eigenvalues $\lambda_1, \lambda_2, \ldots, \lambda_n$ given in non-increasing order, we have:
\begin{equation*}
	\sum_{i=1}^t d_i \leq \sum_{i=1}^t \lambda_i,
\end{equation*}
for any $t=1,2, \ldots, n$, with equality for $t=n$.
\end{theorem}
The Schur--Horn Theorem has been proven in the situation when the diagonal entries $d_1,d_2, \ldots, d_n$ are also given in non-increasing order. Still, one can also use any subset of the diagonal entries, although the resulting lower bound in the above theorem may not be as tight as when using the statement $d_n \leq \ldots \leq d_2 \leq d_1$.

By applying this theorem to both matrices $\bK$ and $\bKc$, we get the following result.
\begin{lemma}\label{th:Schur_Horn_K}
	Let $\bK$ and $\bKc$ be the Gram matrix and its centered counterpart with, respectively, diagonal entries $d_1,d_2, \ldots, d_n$ and $\dc_1,\dc_2, \ldots, \dc_n$; and eigenvalues $\lambda_1, \lambda_2, \ldots, \lambda_n$ and $\lambdac_1, \lambdac_2, \ldots, \lambdac_n$, given in non-increasing order. Then
\begin{equation*}
	\sum_{i=1}^t d_i \leq \sum_{i=1}^t \lambda_i,
	\qquad 
	\sum_{i=1}^t \dc_i \leq \sum_{i=1}^t \lambdac_i,	
\end{equation*}
with equalities for $t=n$.
\end{lemma}
The direct application of the Schur--Horn Theorem separately on each matrix, $\bK$ and $\bKc$, does not give any particular result. The following lemma is a first step towards a connection between the eigenvalues of both matrices, and allows to establish the bounds 
given in Theorem~\ref{th:lambdac}.

\begin{lemma}\label{th:same_eigenvalues}
		Let $\bK=\bA \, \bLambda \, \bA\!^\top$ and $\bKc = \bAc \, \bLambdac \, \bAc\!\!^\top$ be the spectral decompositions of the Gram matrices. Then the spectral decomposition of the matrix $\bA\!^\top\! \bKc \bA$ is $\bA\!^\top\! \bAc \, \bLambdac \, (\bA\!^\top\!\bAc)\!^\top$ and of the matrix $\bAc\!\!^\top\! \bK \bAc$ is $\bAc\!\!^\top\! \bA \,\bLambda \, (\bAc\!\!^\top\!\bA)\!^\top$.
\end{lemma}

It is easy to prove these results, by replacing either $\bKc$ or $\bK$ by its spectral decomposition and verifying that the product of two orthonormal matrices is orthonormal. This lemma shows that the matrix $\bA\!^\top\! \bKc \bA$ has the same eigenvalues as the matrix $\bKc$, and its eigenvectors are the columns of $\bA\!^\top\! \bAc$, which is the matrix whose entries are the inner products between the eigenvectors of $\bK$ and the eigenvectors of $\bKc$ ({\em i.e.,} $\balpha_i^\top\! \balphac_j$). Since both matrices $\bKc$ and $\bA\!^\top\! \bKc \bA$ share the same eigenvalues, we propose next to apply the Schur--Horn Theorem on the latter matrix. 

\begin{theorem}\label{th:lambdac}
The sum of the largest $t$ (for any $t=1,2,\ldots,n$) eigenvalues of the matrix $\bKc$ is lower bounded as follows:
\begin{equation*}
	 \sum_{i=1}^t d_i' \leq \sum_{i=1}^t \lambdac_i,
\end{equation*}
where $d_i'$ is the $i$-th largest value of $\lambda_i + (\|\bmu\|^2 - \tfrac{2}{n} \lambda_i) ~ (\balpha_i^\top \bone)^2$, for $i=1,2, \ldots,n$. Here, $(\lambda_i,\balpha_i)$ is an eigenpair of the matrix $\bK$. This expression provides a lower bound on the largest eigenvalue of $\bKc$, by setting $t=1$:
\begin{equation}\label{eq:th:lambdac_max}
	\max_{i=1,\ldots,n} \lambda_i + (\|\bmu\|^2 - \tfrac{2}{n} \lambda_i) ~ (\balpha_i^\top \bone)^2 
	~~~ \leq ~~~ \lambdac_1.
\end{equation}
\end{theorem}

\begin{proof}
To prove this theorem, we describe the diagonal entries of the matrix $\bA\!^\top\! \bKc \bA$, namely for any $i$:
\begin{align*}
	{\balpha_i^\top\!} &{\bKc \balpha_i} \\
	& = \balpha_i^\top\! \big(\bK - \bP{\bone}\bK - \bK\bP{\bone} + n \|\bmu\|^2 \bP{\bone}\big) \balpha_i \nonumber \\
	& = \lambda_i\balpha_i^\top\! \balpha_i
	- \lambda_i \balpha_i^\top\! \bP{\bone} \balpha_i 
	- \lambda_i \balpha_i^\top\! \bP{\bone} \balpha_i 
	+ n \|\bmu\|^2 \balpha_i^\top\! \bP{\bone} \balpha_i \nonumber\\
	& = \lambda_i\balpha_i^\top\! \balpha_i
	+ (n \|\bmu\|^2 - 2\lambda_i) \, \balpha_i^\top\! \bP{\bone} \balpha_i \nonumber\\
	& = \lambda_i
	+ (\|\bmu\|^2 - \tfrac{2}{n} \lambda_i ) ~ (\balpha_i^\top\! \bone)^2,
\end{align*}
where the first equality follows from expression \eqref{eq:Kc}, the second equality is due to the fact that $(\lambda_i,\balpha_i)$ is an eigenpair of $\bK$, and the last equality follows from the definition of $\bP{\bone}$ given in \eqref{eq:P1}. To conclude the proof, Theorem~\ref{th:Schur_Horn} is applied on the matrix $\bA\!^\top\! \bKc \bA$, after observing from Lemma~\ref{th:same_eigenvalues} that both matrices $\bKc$ and $\bA\!^\top\! \bKc \bA$ share the same eigenvalues.
\end{proof}

This theorem provides a further characterization of the eigenvalues of both matrices $\bK$ and $\bKc$, beyond the relation of their traces given in Lemma~\ref{th:traceK}. Moreover, the latter lemma is obtained as a particular case of our theorem, when $t=n$ where the equality in Theorem~\ref{th:lambdac} holds, namely $\sum_{i=1}^n d_i' = \sum_{i=1}^n \lambdac_i$. To see this, first observe that $\sum_{i=1}^n \lambdac_i = \tr(\bKc)$. Then, we have
\begin{align*}
	 \tr(\bKc) 
	 &=\sum_{i=1}^n d_i' 
\\	 &= \sum_{i=1}^n \lambda_i + (\|\bmu\|^2 - \tfrac{2}{n} \lambda_i ) ~ (\balpha_i^\top\! \bone)^2
\\	 &= \sum_{i=1}^n \lambda_i + \|\bmu\|^2\sum_{i=1}^n (\balpha_i^\top\! \bone)^2 - \tfrac{2}{n} \sum_{i=1}^n \lambda_i (\balpha_i^\top\! \bone)^2
\\	 &= \tr(\bK) - n\|\bmu\|^2
\end{align*}
where the expression of $\|\bmu\|^2$ follows from \eqref{eq:norm_mean}, and we have used $\sum_{i=1}^n (\balpha_i^\top\! \bone)^2 = \sum_{i=1}^n \bone\!^\top\! \balpha_i \balpha_i^\top\! \bone = \bone\!^\top\! \bA \bA\!^\top\! \bone = \bone\!^\top\! \bone = n$. This illustrates the tightness of the derived bounds.

Furthermore, it is worth noting that the largest value of $\lambda_i + (\|\bmu\|^2 - \tfrac{2}{n} \lambda_i) ~ (\balpha_i^\top \bone)^2$ needs not to be given by the largest eigenvalue $\lambda_1$ and the corresponding eigenvector $\balpha_1$. To see this, we show that $\|\bmu\|^2 - \tfrac{2}{n} \lambda_i < 0$ in this case. To this end, we apply the celebrated Courant-Fischer Theorem on the matrix $\bK$, which states that $\lambda_1 = \max_{\bv} \frac{\bv\!^\top\! \bK \bv}{\bv\!^\top\! \bv}$. As a consequence, $\lambda_1$ is larger or equal to the special case when $\bv=\bone$. Consequently $\lambda_1 \geq \tfrac1n \bone\!^\top\! \bK \bone$, and therefore we have $\|\bmu\|^2 < \tfrac{2}{n} \lambda_1$ from \eqref{eq:norm_mean}.

Finally, Theorem~\ref{th:lambdac} reveals the terms $\lambda_i (\balpha_i^\top\! \bone)^2$ in the lower bound. It turns out that these terms are the building blocks of the entropy estimate, as given in \eqref{eq:keca} for the (kernel) entropy component analysis. 

\subsubsection{Eigenvectors of $\bKc$}\label{sec:eigenvectors}

We propose to go further beyond the analysis of the eigenvalues as given so far. In this section, we study the eigenvectors of the centered Gram matrix. The following theorem provides insights on the eigenvectors of the matrix $\bKc$.

\begin{theorem}\label{th:sum_constraint}
	For any eigenvector $\balphac_j$ of the matrix $\bKc$ associated to a non-zero eigenvalue, its entries sum to zero, namely $\balphac_j^\top\! \bone = 0$ for any $\lambdac_j \neq 0$.
\end{theorem}

\begin{proof}
The proof is straightforward, with
\begin{equation*}
	\balphac_i^\top\! \bone 
	= \tfrac{1}{\lambdac_i} \balphac_i^\top\! \bKc \bone
	= \tfrac{1}{\lambdac_i} \balphac_i^\top\! (\bI - \bP{\bone}) \bK (\bI - \bP{\bone}) \bone 
	= \bzero,
\end{equation*}
where the first equality follows from the eigenproblem \eqref{eq:eig_Kc} and the last equality is due to \eqref{eq:I_P_1}.
\end{proof}
It is therefore easy to see that $\bP{\bone} \balphac_j = \bzero$ for any eigenvector of $\bKc$ associated to a non-zero eigenvalue, and we have in its dual form $(\bI - \bP{\bone}) \balphac_j = \balphac_j$.

\begin{theorem}\label{th:box_constraint}
	All entries of any eigenvector $\balphac_j$ of the matrix $\bKc$ of the centered data are bounded with
	\begin{equation*}
		-1 \leq \balphac_j \leq 1,
	\end{equation*}
	where inequalities are applied element-wise.
\end{theorem}

\begin{proof} 
It is well known that $\|\balphac_j\|_\infty \leq \|\balphac_j\| \leq \|\balphac_j\|_{_1} \leq \sqrt{n} \|\balphac_j\|_\infty$, where $\|\cdot\|_\infty$ is the supremum norm which takes the largest absolute value of the vector's entries. Since eigenvectors have a unit norm, we get the pair of inequalities.
\end{proof}

By combining Theorems \ref{th:sum_constraint} and \ref{th:box_constraint}, we have that each eigenvector of the centered Gram matrix verifies the sum-to-one and the boxed constraints. These constraints are equivalent to the well-known constraints of {\em SVM}. One can also describe data-driven bounds when a normalization is operated, as given in \eqref{eq:kpca_normalization} for instance. In this case, the normalization $\|\balphac_i\|^2 = {1}/{\lambdac_i}$ yields a modification on the box constraints, since $\|\balphac_j\|_\infty \leq \|\balphac_j\| = 1/\sqrt{\lambdac_i}$. The latter can also be upper bounded by using the eigenvalues of $\bK$, thanks to the interlacing property derived in Theorem~\ref{th:ineq_lambda}.

The eigenvectors of the Gram matrix are seldom used directly, but often considered to define relevant axes. This is shown in Section~\ref{sec:PCA} for the principal component analysis, and in Section~\ref{sec:KECA} for the (kernel) entropy component analysis. In either methods, the relevant axes are determined by a weighted linear combination of the data, the weights being the eigenvectors of the Gram matrix, up to a normalization factor. By considering the normalization given in \eqref{eq:kpca_normalization}, we get 
\begin{equation}\label{eq:wj_alphaj}
	\bw_j = \tfrac{1}{\sqrt{\lambda_i}}\bX \balpha_j,
\end{equation}
and its centered counterpart $\bwc_j = \bXc \balphac_j /\sqrt{\lambdac_i}$. Therefore, the projection of the data on either axes is:
\begin{equation*}
	\bX\!^\top\! \bw_j  
	= \tfrac{1}{\sqrt{\lambda_i}} \bX\!^\top\!\!\bX \balpha_j
	= \tfrac{1}{\sqrt{\lambda_i}} \bK \balpha_j
	= \sqrt{\lambda_i} \, \balpha_j
\end{equation*}
and likewise $\bXc\!^\top\! \bwc_j=\sqrt{\lambdac_i} \balphac_j$. By using the normalization \eqref{eq:kpca_normalization2}, we get a unit variance along the respective axes, with
\begin{equation}\label{eq:wj_alphaj2}
	\bX\!^\top\! \bw_j = \balpha_j, \qquad \text{and} \qquad \bXc\!^\top\! \bwc_j = \balphac_j.
\end{equation}
The analysis of these axes is derived next, by examining the outer product matrices.

\subsection{Outer product -- the covariance matrix $\bCc$}\label{sec:C}

The relation between the covariance matrix $\bCc= \frac1n \bXc \bXc\!\!^\top$, and $\bCn= \frac1n \bX \bX\!^\top$, {\em i.e.,} the second-order non-central moment matrix, is given by
\begin{equation}\label{eq:C_rankone}
\bCc 
		= \tfrac1n \bX (\bI - \bP{\bone})\bX\!^\top 
		= \bCn - \bmu \bmu\!^\top.
\end{equation}
Since $\frac{1}{\|\bmu\|^2} \bmu \bmu\!^\top$ denotes the $d$-by-$d$ projection matrix onto the vector mean $\bmu$ then, by 
analogy with the definition of $\bKc$ in \eqref{eq:Kc}, we have here a simpler expression. We can therefore revisit all the results given in Section~\ref{sec:K} to describe relations between the eigenvectors of $\bCn$ and $\bCc$. The eigenvalues of these matrices still satisfy the interlacing theorems given in Section \ref{sec:interlacing}, since the eigenvalues of $\bCn$ and $\bCc$ are respectively $\frac1n \lambda_1, \frac1n \lambda_2, \frac1n \lambda_3 \ldots$, and $\frac1n \lambdac_1, \frac1n \lambdac_2, \frac1n \lambdac_3 \ldots$. By analogy with Lemma~\ref{th:traceK}, we have $\tr(\bCc) = \tr(\bCn) - \|\bmu\|^2$.

The following lemma provides an expression essential to our study.
\begin{lemma}\label{th:delta_lambda}
	For any eigenpair $(\lambda_i,\bw_i)$ of $\bCn$ and any eigenpair $(\lambdac_i,\bwc_i)$ of $\bCc$, we have the following relation with the mean vector $\bmu$:
	\begin{equation*}
	 \left(\lambda_i - \lambdac_j \right) \bw_i^\top \bwc_j 
	 = n ~ \bw_i^\top \bmu ~~ \bwc_j^\top \bmu.
	\end{equation*}
\end{lemma}

\begin{proof}
Since $\bwc_j$ is an eigenvector of $\bCc$, we have from expression \eqref{eq:wci_eigen}: 
	$\bw_i^\top \bwc_j = \tfrac{n}{\lambdac_j} \, \bw_i^\top \bCc \, \bwc_j$. 
By substituting the definition of $\bCc$ from \eqref{eq:C_rankone}, we get:
\begin{align*}
	{\bw_i^\top \bwc_j} 
	&=\tfrac{n}{\lambdac_j} \, \bw_i^\top (\bCn - \bmu \bmu^\top) \, \bwc_j\\
	&=\tfrac{n}{\lambdac_j} \, \bw_i^\top \bCn \, \bwc_j 
	- \tfrac{n}{\lambdac_j} \, \bw_i^\top \bmu \, \bmu^\top \, \bwc_j.
\end{align*}
The first term in the right-hand-side can be simplified, since $\bw_i$ is an eigenvector of $\bCn$, thus $\bw_i^\top \bCn = \tfrac{1}{n} \lambda_i \bw_i^\top$. 
\end{proof}

In the above expression, $\bw_i^\top \bmu$ corresponds to the so-called mean score over the $i$-th principal component $\bw_i$ in non-centered PCA, since $\bw_i^\top \bmu = \bw_i^\top \tfrac1n \bX \bone$ corresponds to the mean of the {\em score vector} $\bX\!^\top\! \bw_i$. Moreover, we have 
\begin{equation}\label{eq:wi_mu}
	\bw_i^\top \bmu 
	= \tfrac{1}{\sqrt{\lambda_i}} \balpha_i^\top \bX\!^\top\! \, \tfrac1n \bX \bone
	= \tfrac{1}{n\sqrt{\lambda_i}} \balpha_i^\top \bK \bone
	= \tfrac{\sqrt{\lambda_i}}{n} \, \balpha_i^\top \bone,
\end{equation}
where the first equality follows from \eqref{eq:wj_alphaj} and the last one is due to the fact that $\balpha_i$ is an eigenvector of $\bK$. Once again, we get the main building blocks of the entropy estimate \eqref{eq:keca}. Therefore, it is easy to see that the entropy estimate from the ECA is simply
\begin{equation*}
	\int \hat{p}(\bx)^2 d\bx 
	=\tfrac{1}{n^2} \sum_{i=1}^n \lambda_i (\balpha_i\!^\top\! \bone)^2
	= \|\bW\!^\top\! \bmu\|^2.
\end{equation*}

Expression \eqref{eq:wi_mu} illustrates that Theorem~\ref{th:lambdac} can be rewritten in terms of $\bw_i^\top \bmu$. In the following theorem, we derive more effective expressions 
thanks to the simplified definition of $\bCc$.

\begin{theorem}\label{th:lambdac_C}
Theorem~\ref{th:lambdac} can be expressed as follows: 
$$\sum_{i=1}^t d_i' \leq \sum_{i=1}^t \lambdac_i,$$
where $d_i'$ is the $i$-th largest value of $\lambda_i - n (\bw_i^\top \bmu)^2$, for $i=1,2, \ldots,n$. In this expression, $(\tfrac1n \lambda_i,\bw_i)$ is an eigenpair of the matrix $\bCn$. Therefore, we have:
\begin{equation*}
	\max_{i=1,\ldots,n} \lambda_i - n (\bw_i^\top \bmu)^2
	~~  \leq ~~ \lambdac_1.
\end{equation*}
\end{theorem}

\begin{proof}
To prove this result, we use essentially the same steps given in the proof of Theorem~\ref{th:lambdac}, by considering the diagonal entries of the matrix $\bW\!^\top\! \bCc \bW$, with
\begin{align*}
	{\bw_i^\top\! \bCc \bw_i}
	= \bw_i^\top\! (\bCn - \bmu \bmu\!^\top\!) \bw_i 
	= \tfrac{\lambda_i}{n} \bw_i^\top\! \bw_i - \bw_i^\top\! \bmu \, \bmu\!^\top\! \bw_i,
\end{align*}
and therefore $\bw_i^\top\! \bCc \bw_i = \tfrac{\lambda_i}{n} - (\bw_i^\top \bmu)^2$. Finally, we apply Theorem~\ref{th:Schur_Horn} on the matrix $\bW\!^\top\! \bCc \bW$, and observe by analogy to Lemma~\ref{th:same_eigenvalues} that both matrices $\bCc$ and $\bW\!^\top\! \bCc \bW$ share the same eigenvalues, {\em i.e.,} $\tfrac1n\lambdac_1, \tfrac1n\lambdac_2, \ldots, \tfrac1n\lambdac_n$. 
\end{proof}


The previous theorem has several important consequences. The following theorem states that the mean vector is close to the eigenvector associated to the largest eigenvalue of $\bCn$.
\begin{theorem}\label{th:w1_m_ineq}
We have the following lower bound on the inner product between the mean vector $\bmu$ and the first eigenvector of $\bCn$:
\begin{equation*}
	\lambda_1 - \lambdac_1 \leq n (\bw_1^\top \bmu)^2.
\end{equation*}
\end{theorem}

\begin{proof}
From Theorem~\ref{th:lambdac_C}, $\lambda_i - \lambdac_1 \leq n (\bw_i^\top \bmu)^2$ for any $i$. This inequality can be investigated only when the left-hand-side is non-negative. As shown from the interlacing property in Theorem~\ref{th:ineq_lambda}, $\lambdac_1 \leq \lambda_i$ if and only if $i=1$. By considering this case, we conclude the proof.
\end{proof}

The following theorem shows that the eigenvectors associated to the largest eigenvalue of each matrix $\bCn$ and $\bCc$, cannot be arbitrary different.
\begin{theorem}\label{th:innerproducts_ineq}
We have the following lower bound on the inner product between the first eigenvector of each of $\bCn$ and $\bCc$:
\begin{equation*}
	  \frac{(\bwc_1^\top \bmu)^2}{\|\bmu\|^2}
	\leq (\bw_1^\top \bwc_1)^2.
\end{equation*}
\end{theorem}

\begin{proof}
From Theorem~\ref{th:w1_m_ineq}, and by replacing the left-hand-side by the expression given in Lemma~\ref{th:delta_lambda}, we get
\begin{equation*}
	\frac{\bw_1^\top \bmu ~~ \bwc_1^\top \bmu}{\bw_1^\top \bwc_1} \leq  (\bw_1^\top \bmu)^2.
\end{equation*}
By squaring and simplifying by $(\bw_1^\top \bmu)^2$, we obtain:
\begin{equation*}
	\frac{(\bwc_1^\top \bmu)^2}{(\bw_1^\top \bmu)^2} 
		\leq  (\bw_1^\top \bwc_1)^2.
\end{equation*}
The above denominator can be upper bounded thanks to the Cauchy-Schwarz inequality, with 
	$(\bw_1^\top \bmu)^2 \leq \|\bw_1\|^2 \|\bmu\|^2 = \|\bmu\|^2$, 
due to the normalization of the eigenvectors. By combining these results, this concludes the proof.
\end{proof}

This theorem has an immediate result which shows that $\bwc_1$ is {\em closer} to $\bw_1$ than to $\bmu$. This property is illustrated in the following corollary.
\begin{corollary}\label{th:cosine_ineq}
The cosine of the angle between the first eigenvectors of $\bCn$ and $\bCc$ 
is lower bounded as follows:
\begin{equation*}
	\cos(\bwc_1, \bmu)^2 \leq \cos(\bwc_1,\bw_1)^2.
\end{equation*}
\end{corollary}

\begin{proof}
The proof is straightforward from Theorem~\ref{th:innerproducts_ineq} and the definition of the inner product with $\|\bw_1\| = \|\bwc_1\|=1$, namely
$\bw_1^\top \bwc_1 = \cos(\bw_1, \bwc_1)$ and $\bw_1^\top \bmu = \|\bmu\| \cos(\bw_1, \bmu)$.
\end{proof}

We conclude this section by giving a summary of the relations obtained between the eigenvectors associated to the largest eigenvalues of the covariance matrix and its non-centered counterpart. In the non-centered case, we see from Theorem~\ref{th:w1_m_ineq} that the eigenvector $\bw_1$ of $\bCn$ tends to be collinear with the the mean vector $\bmu$. Now, consider the case when data are centered, which leads to the first eigenvector $\bwc_1$ of $\bCc$. Theorem~\ref{th:innerproducts_ineq} and Corollary~\ref{th:cosine_ineq} provide inequalities that measure the fact that $\bwc_1$ is ``closer'' to $\bw_1$ than to $\bmu$.

\subsubsection{Connections to the work of Cadima and Jolliffe}\label{sec:connections}

The above results, obtained by confronting the covariance matrix and its non-centered counterpart, corroborate the work of Cadima and Jolliffe in \cite{PCAuncenter}. In the latter, the eighth property in Proposition 3.1 gives a result equivalent to the above Lemma~\ref{th:delta_lambda}, however our proof is much shorter and significantly simpler than in \cite[proof that spans nearly all the page 499]{PCAuncenter}. Likewise, the ninth property in Proposition 3.1 is equivalent to Theorem~\ref{th:lambdac_C}, while our proof is slightly simpler.

Finally, Theorem~\ref{th:innerproducts_ineq} and Corollary~\ref{th:cosine_ineq} provide bounds that do not depend on the relation between $\bw_1$ and $\bmu$, as opposed to the fourteenth property in Proposition 3.1 in \cite{PCAuncenter}. In our case, we have more comprehensive expressions with simpler bounds, thus offering a straightforward interpretation.

\section{Beyond conventional centering}\label{sec:extensions}

In this section, we show that several research activities can take advantage of our study, apart from the conventional centering issue and beyond the scope of bridging the gap between PCA and ECA.


\subsection{Weighted mean shift}

The issue of centering the data has also been investigated with the use of a vector other than the mean of the data. Weighted means provide a generalization of the conventional mean. They have been commonly studied in the literature, such as in statistics with population studies \cite{Meier1953}. More recently, a weighted mean is considered in \cite{Higuchi:2004} to derive a robust PCA algorithm. In \cite{SuzukiHSSF13}, the authors study the use of a weighted mean in a k-nearest neighbor algorithm, in order to reduce hubs\footnote{A hub is a sample that is very similar to many other samples of the dataset. Hubs emerge from the curse of dimensionality, and tend to be close to the data mean, {\em i.e.,} centroid. See \cite{Radovanovic2010} for more details on the concept of hubs in machine learning.}.

We propose to extend the our study to the issue of a weighted mean. Let $\bomega$ be a weight vector such as $\bomega\!^\top\! \bone =1$, and let $\bmu_\bomega=\bX \bomega$ be the corresponding weighted mean. The matrix
\begin{equation}\label{eq:Pw}
	\bP{\bomega} =\bomega \bone\!^\top
\end{equation}
maps the data such that their weighted mean becomes zero, with ${\bX \bP{\bomega}} = \bmu_\bomega \bone\!^\top$. This matrix defines a projection map, since it is idempotent ({\em i.e.}, ${\bP{\bomega}}\!\!^2 = \bP{\bomega}$), but it is not necessary orthogonal. To define an orthogonal projection, the matrix needs to be symmetric, which means that $\bomega = \frac1n \bone$ and therefore we get the particular case of the conventional mean studied so far.

As shown in the following, it turns out that this generalization of the projection can be easily studied with the analysis of the inner product matrices, as given in Section~\ref{sec:K}. Unfortunately, the analysis of the covariance matrix is no longer as easy as in Section~\ref{sec:C}. The main difficulty raises from the non-symmetric property of the matrix $\bP{\bomega}$ in the general case, due to the relaxation of the orthogonality in the projection.

Before proceeding, we revisit relations \eqref{eq:I_P_1}--\eqref{eq:trIP_M_IP} in the light of this general definition, as follows:
\begin{eqnarray*}
	&(\bI - \bP{\bomega})^\top \bomega = \bomega\!^\top (\bI - \bP{\bomega}) = \bzero;&
\\	&{\bP{\bomega}}\!\!^\top \! \bM \bP{\bomega} = \bone \bomega\!^\top \! \bM \bomega \bone\!^\top = n \, (\bomega\!^\top \!\bM \bomega) \, \bP{\bone};&
\\	&\tr({\bP{\bomega}}\!\!^\top \! \bM) = \tr(\bM\!^\top \! \bP{\bomega}) = \tfrac1n\tr({\bP{\bomega}}\!\!^\top \! \bM \bP{\bomega}) = \bomega\!^\top \!\bM \bone
.&
\end{eqnarray*}

By substituting $\bM$ with $\bK$ in the above expressions, we get $\bomega\!^\top \! \bK \bomega = \bomega\!^\top\! \bX\!^\top\! \bX \bomega = \|\bmu_\bomega\|^2$. Therefore, the definition of $\bKc$ becomes
\begin{align}\label{eq:Kcw}
	{\bKc} 
		&= \bK - {\bP{\bomega}}\!\!^\top \bK - \bK\bP{\bomega} + n \|\bmu_\bomega\|^2 \bP{\bone}.
\end{align}
Lemma~\ref{th:traceK} becomes $\tr(\bKc) = \tr(\bK) - 2n \bmu_\bomega^\top \bmu + n \|\bmu_\omega\|^2$, where we have used $\bomega\!^\top \!\bK \bone = \bomega\!^\top\! \bX\!^\top\! \bX \bone = n \, \bmu_\bomega^\top \bmu$.

The analysis of the eigenvalues remains unchanged in the weighted mean case, including the interlacing property as given in Theorem~\ref{th:ineq_lambda}. The only difference lies in the bounds proposed in Theorem~\ref{th:lambdac}. The general results are derived from expression \eqref{eq:Kcw} as follows:
\begin{equation*}
	\max_{i=1,\ldots,n} \lambda_i
	+ \big(\|\bmu\|^2 \, \balpha_i^\top\! \bone - 2 \lambda_i \, \balpha_i^\top\! \bomega \big) \, \balpha_i^\top\! \bone
	~~~ \leq ~~~ \lambdac_1.
\end{equation*}
It is also easy to verify that the eigenvectors satisfy $\balphac_j^\top\! \bomega = 0$ for any non-zero eigenvalue.

The analysis of the covariance matrix is more complicated than the study derived in Section~\ref{sec:C}. This is due to the resulting expression of the covariance matrix in the general case of a weighted mean, with\begin{align*}
\bCc 
 & = \tfrac1n \bX (\bI - \bP{\bone})(\bI - \bP{\bone}\!^\top)\bX\!^\top
\\ & = \bCn - \bmu_\bomega \bmu\!^\top - \bmu \bmu_\bomega^\top
		+ \bmu_\bomega \bmu_\bomega^\top.
\end{align*}
Still, one can derive several results. For instance, the lower bound in Theorem~\ref{th:lambdac_C} becomes:
\begin{equation*} 
	\max_{i=1,\ldots,n} \lambda_i 
	- 2n \, \bw_i^\top \bmu_\bomega \; \bw_i^\top \bmu 
	+ n (\bw_i^\top \bmu_\bomega)^2
\leq \lambdac_1.
\end{equation*}

\subsection{Rank-one update of the covariance matrix}

As given in expression \eqref{eq:C_rankone}, the matrix $\bCc$ is a special case of the rank-one update of the matrix $\bC$. It turns out that the study given in Sections \ref{sec:K} and \ref{sec:C} can be extended to any rank-one update of the covariance matrix. Such update is of great interest in covariance matrix adaptation within machines that rely on Gaussian random variations, such as evolutionary strategies \cite{suttorp2009ecm} and ensemble optimization \cite{Fonseca2013}. In \cite{Djuric2010}, the author provides connections of these genetic machines to Monte Carlo-based methods, including particle filtering and population Monte Carlo. Without lost of generality\footnote{For the sake of clarity, we examine the rank-one update of the covariance matrix. One could also study the rank-one update of the Gram matrix. In this case, simply replace $\bC_t$, $\bw_{i,t}$ and $\tfrac{1}{n}\lambda_{i,t}$, with $\bK_t$, $\balpha_{i,t}$ and $\lambda_{i,t}$.}, this section presents the issue of covariance matrix adaptation in evolutionary strategies \cite{Muller2010}. See also \cite{Meyer-Nieberg2006,Kramer2010} for a survey.

The re-sampling techniques solve hard optimization problems by generating a set of candidate solutions. The performance highly depends on the population's distribution under investigation. Evolutionary strategies provide an elegant approach to derive (quasi) parameter-free techniques for the user. This principle of self-adaptation allows to adjust the distribution in the direction of more relevant regions in the search space.

Let $\bv_t$ be a candidate solution generated from a zero-mean Gaussian distribution with covariance matrix $\bC_t$. The latter is adapted according to the relevance of $\bv_t$ in the optimization problem. Let $\nu_t \in \; ] 0 , 1 [$ be a parameter that measures this relevance, where high values correspond to promising pertinent fitness progress. The rank-one update rule of the covariance matrix at iteration $t$ is given by
\begin{equation}\label{eq:C_t}
	\bC_{t+1} = (1-\nu_t) \,\bC_{t} + \nu_t \, \bv_t \bv_t^\top.
\end{equation}
This rule allows to adjust the distribution towards the zero-mean Gaussian distribution with covariance matrix $\bv_t \bv_t^\top$, namely the distribution with the highest probability to generate $\bv_t$
among all zero-mean Gaussian distributions.

We show next that one can take advantage of the mathematical statements presented in Section~\ref{sec:C} in order to provide new insights to the update rule \eqref{eq:C_t}. Let $(\frac1n \lambda_{i,t},\bw_{i,t})$ be the $i$-th eigenpair of $\bC_{t}$, namely
\begin{equation}\label{eq:eigenCt}
	\bC_{t} \, \bw_{i,t} = \tfrac{1}{n} \lambda_{i,t} \, \bw_{i,t}.
\end{equation}

Firstly, the variance of the data can be measured with the sum of eigenvalues of the covariance matrix, which is given by its trace thanks to the relation~\eqref{eq:trace}. By following the same derivations as in Lemma~\ref{th:traceK}, we get the following relation $\sum_{i=1}^n \lambda_{i,t+1} = (1-\nu_t) \sum_{i=1}^n \lambda_{i,t} +  n \, \nu_t \, \|\bv_t\|^2$. From expression~\eqref{eq:C_t}, Lemma~\ref{th:delta_lambda} becomes:
\begin{equation}\label{eq:delta_lambda_t}
	 \left(\lambda_{j,t+1} - (1-\nu_t)\lambda_{i,t} \right) \bw_{i,t}^\top \bw_{j,t+1}  = n \, \nu_t \, \bw_{i,t}^\top \bv_t ~ \bw_{j,t+1}^\top \bv_t.
\end{equation}

Bounds on the eigenvalues of the covariance matrix can be easily derived,  by following the same steps given in the proof of Theorem~\ref{th:lambdac_C}. 
Therefore, we have
\begin{equation*}
	\max_{i=1,\ldots,n} (1-\nu_t)\lambda_{i,t} + n\nu_t (\bw_{i,t}^\top \bv_t)^2
	~~ \leq ~~ \lambda_{1,t+1}.
\end{equation*}

In order to study the impact of the update rule on the eigenvectors of the covariance matrice, we revisit Theorems~\ref{th:w1_m_ineq}-\ref{th:innerproducts_ineq} and Corollary \ref{th:cosine_ineq}. From the above expression, we have for any $i=1,2, \ldots, n$:
\begin{equation*}
	\lambda_{1,t+1} - (1-\nu_t)\lambda_{i,t}  \geq n\nu_t (\bw_{i,t}^\top \bv_t)^2.
\end{equation*}
By injecting expression \eqref{eq:delta_lambda_t} for $j=1$, we get
\begin{equation*}
	\frac{n \, \nu_t \, \bw_{i,t}^\top \bv_t ~ \bw_{1,t+1}^\top \bv_t}{\bw_{i,t}^\top \bw_{1,t+1}}  \geq n\nu_t (\bw_{i,t}^\top \bv_t)^2.
\end{equation*}
By squaring and simplifying by $(\bw_{i,t}^\top \bv_t)^2$, we obtain
\begin{equation*}
	(\bw_{i,t}^\top \bw_{1,t+1})^2 \leq \frac{(\bw_{1,t+1}^\top \bv_t)^2}{(\bw_{i,t}^\top \bv_t)^2},
\end{equation*}
and equivalently
\begin{equation*}
	\cos(\bw_{i,t}, \bw_{1,t+1})^2 \leq \frac{\cos(\bw_{1,t+1}, \bv_t)^2}{\cos(\bw_{i,t}, \bv_t)^2}.
\end{equation*}
This bound shows that the first eigenvector of $\bC_{t+1}$ forms a greater angle with all the eigenvectors of $\bC_{t}$ than with the vector $\bv_t$. This result, independent of the value of the parameter $\nu_t$, illustrates the diversity introduced by applying the update rule \eqref{eq:C_t}.

\subsection{Multidimensional scaling}

Multidimensional scaling (MDS) is a well-known dimensionality reduction technique that seeks to preserve pairwise distances or dissimilarity measures \cite{MDS}. The problem is to estimate all $\bx_1, \bx_2, \ldots, \bx_n$ from their available distances, denoted $\| \bx_i - \bx_j \|$ between $\bx_i$ and $\bx_j$. By expanding this expression, we get $\| \bx_i - \bx_j \|^2 
= \bx_i^\top\! \bx_i + \bx_j^\top\! \bx_j - 2 \bx_i^\top\! \bx_j$. Therefore, one can define an inner product from distances, with $\bx_i^\top\! \bx_j = -\tfrac12 ( \|\bx_i - \bx_j\|^2 -  \|\bx_i\|^2 - \|\bx_j\|^2)$, or equivalently in matrix form
\begin{equation*}
	\bX\!^\top\! \bX = \bDelta +\tfrac12 \bdelta \bone\!^\top + \tfrac12 \bone\bdelta\!^\top,
\end{equation*}
where $\bDelta$ is the matrix of entries $-\tfrac12 \| \bx_i - \bx_j \|^2$ and $\bdelta$ is the column vector whose $i$-th entry is $\|\bx_i\|^2$. In order to remove the indeterminacy with respect to translation, the inner product of the centered data is considered. Thus, the double centering from \eqref{eq:Kc_Kn} gives us:
\begin{align}\label{eq:MDS_Kc_Delta}
	{\bKc} 
		&= (\bI - \bP{\bone}) \bX\!^\top\! \bX (\bI - \bP{\bone}) \nonumber \\
		&= (\bI - \bP{\bone})( \bDelta + \tfrac12 \bdelta \bone\!^\top + \tfrac12 \bone\bdelta\!^\top) (\bI - \bP{\bone})\nonumber\\
		&=  (\bI - \bP{\bone}) \bDelta (\bI - \bP{\bone}),
\end{align}
where the last equality follows from \eqref{eq:I_P_1}. An eigendecomposition of this matrix provides the relevant axes to describe the samples. This is the classical MDS. Next, we study the MDS in the light of our work.

Expression \eqref{eq:MDS_Kc_Delta} is similar to the double centering of the Gram matrix in \eqref{eq:Kc_Kn}. While, by construction, $\bKc$ in both expressions is a Gram matrix, there is however one major difference: $\bDelta$ is not a positive definite matrix. It turns out that the corresponding function $\kappa(\bx_i,\bx_j) = -\frac12 \| \bx_i - \bx_j \|^2$ is a conditionally positive definite kernel \cite{distancetrick}. Next, we define this principle and give some properties, before studying its impact on the mathematical statements in this paper.

A conditionally positive definite kernel is a symmetric function that satisfies the inequality \eqref{eq:pd_kernel} for any $\bbeta$ such that $\bbeta^\top\! \bone=0$. This is the case of $\kappa(\bx_i,\bx_j) = -\frac12 \| \bx_i - \bx_j \|^2$, since we have for any $\bbeta^\top\! \bone = 0$:
\begin{equation*}
	\bbeta\!^\top\!\! \bDelta \bbeta  = \bbeta\!^\top\!( 2\bX\!^\top\!\! \bX - \bdelta \bone\!^\top \!- \bone\bdelta\!^\top \! ) \bbeta = 2 \bbeta\!^\top\!\! \bX\!^\top\!\! \bX \bbeta = 2 \| \bX \bbeta \|^2 \! \geq \! 0.
\end{equation*}
In \cite{distancetrick}, the author provides a thorough description of conditionally positive definite kernels, and argues that they are ``as good as'' positive definite kernels whenever a translation invariant problem is investigated, such as in PCA and SVM. It is worth noting that one can include a positive bias $b$ large enough such that $\kappa(\bx_i,\bx_j) + b$ is positive definite, thus eliminating the term associated to the negative eigenvalue.

We return now to the main issue of this section, which is the analysis of the relations between the matrices $\bKc$ and $\bDelta$. It turns out that one can take advantage of most of the mathematical statements derived in Section~\ref{sec:K} for this purpose, as illustrated next by substituting $\bK$ with $\bDelta$. To this end, we write expression \eqref{eq:MDS_Kc_Delta} as follows:
\begin{align*}
	{\bKc} 
	&=  \bDelta - \bP{\bone} \bDelta - \bDelta\bP{\bone} + \tfrac{\bone\!\!^\top\!\! \bDelta \bone}{n} \, \bP{\bone}.
\end{align*}
This illustrates the analogy with \eqref{eq:Kc_Kn2}, where
\begin{equation*}
	\bone\!^\top\!\! \bDelta \bone = - \frac12 \sum_{i,j=1}^{n} \|\bx_i - \bx_j\|^2.
\end{equation*}

Moreover, since all diagonal entries of $\bDelta$ are null, then its trace is null, as well as the sum of its eigenvalues. Lemma~\ref{th:Schur_Horn_K}
and Lemma~\ref{th:traceK} provide new insights. The former, namely the Schur-Horn Theorem given in Theorem~\ref{th:Schur_Horn} applied for the matrix $\bDelta$, shows that $0 \leq \sum_{i=1}^t \lambda_i$ for all $t=1,2, \ldots, n$, with equality for $t=n$, that is $\lambda_n = - \sum_{i=1}^{n-1} \lambda_i$. Lemma~\ref{th:traceK} shows the variability of the data with $\sum_{i=1}^n \lambdac_i = 
 -\frac1n \bone\!\!^\top\!\! \bDelta \bone = \frac{1}{2n} \sum_{i,j=1}^{n} \|\bx_i - \bx_j\|^2$.

An analysis of the distribution of the eigenvalues of each of $\bKc$ and $\bDelta$ is given by considering once again the Separation Theorem given in Theorem~\ref{th:separation}. To this end, let $\bM = \bDelta$, $\bP{\mathrm{left}} = \bP{\mathrm{right}} = (\bI - \bP{\bone})$, where $r(\bP{\mathrm{left}})=r(\bP{\mathrm{right}})=n-1$, and thus $t=2$. This leads to the following pair of inequalities:
\begin{equation*}
	\sigma_{j+2}(\bDelta) \leq \sigma_{j}(\bKc) \leq \sigma_{j}(\bDelta).
\end{equation*}
The resulting inequalities are not as tight as the ones given in Theorem~\ref{th:ineq_lambda}, due to the use of the decomposition $\bKc = \bXc^\top\! \bXc$ in the latter case. Similar tight inequalities may be derived when one assumes that such decomposition is valid for $\bKc$ defined in \eqref{eq:MDS_Kc_Delta}.

Bounds on the eigenvalues of $\bDelta$ and $\bKc$ can be derived with the help of the Schur--Horn Theorem given in Theorem~\ref{th:Schur_Horn}. By virtue of Lemma~\ref{th:same_eigenvalues}, one can describe results by following the same steps in the proof of Theorem~\ref{th:lambdac}. Thus, a lower bound on the largest eigenvalue of $\bKc$ is given by 
\begin{equation*} 
	 \max_{i=1,\ldots,n} \lambda_i
	+ (\tfrac{\bone\!\!^\top\!\! \bDelta \bone}{n^2} - \tfrac{2}{n} \lambda_i ) ~ (\balpha_i^\top \bone)^2
	~~ \leq ~~ \lambdac_1,
\end{equation*}
where $(\lambda_i,\balpha_i)$ is an eigenpair of $\bDelta$.


When applied on noisy data in practice, the MDS technique considers the factorization of the resulting matrix $\bKc$, such that $\bKc = \bAc \bLambdac \bAc\!\!\!^\top$ where only the largest non-negative eigenvalues are retained. From this expression, one defines a Gram matrix by setting $\bXc = (\bLambdac)\!^{\frac12} \!\bAc\!\!\!^\top$. This construction leads to uncorrelated data, since 
\begin{equation*}
	\bCc = \tfrac1n \bXc \bXc^\top = \tfrac1n (\bLambdac)\!^{\frac12} \bAc\!\!\!^\top\! \bAc (\bLambdac)\!^{\frac12} = \tfrac1n \bLambdac,
\end{equation*}
and therefore, the analysis of the covariance matrix given in Section~\ref{sec:C} is no longer required.

\subsubsection*{Scaling the data}

We conclude this section by some interesting properties borrowed from \cite{Sahbi2007} and naturally completes our work. We describe some interesting properties of a matrix $\bDelta$ obtained from the kernel $\kappa(\bx_i,\bx_j) = -\frac12 \| \bx_i - \bx_j \|^2$. Consider scaling the data with some positive factor $\xi$. Since $\kappa(\xi\bx_i,\xi\bx_j)=\xi^2\kappa(\bx_i,\bx_j)$, then the corresponding matrix is $\bDeltas = \xi^2 \bDelta$. From this relation, it is easy to see that both matrices $\bDelta$ and $\bDeltas$ share the same eigenvectors, while any eigenvalue $\lambda_i$ of the former defines the eigenvalue $\xi^2\lambda_i$ of the latter. 
By considering the normalization given in \eqref{eq:kpca_normalization2} with $\|\balpha_i\| = \frac{1}{\lambda_i}$, we obtain from \eqref{eq:wj_alphaj2}: $\bX\!^\top\! \bw_j = \balpha_j = \bX_{\!\xi}\!\!\!^\top\! \bw_{\!\xi \, j}$. Therefore, projections onto axes defined by either PCA or ECA provide scale-invariant features, as show here within the MDS approach. These results extends the work in \cite{Sahbi2007} where only PCA is studies.

\section{Experimental results}

All established mathematical statements can be easily verified. To show this, we consider the well-known {\em iris} dataset (available at the UCI Machine Learning Repository), which has been extensively studied in the pattern recognition literature since Fisher's seminal paper \cite{Fis36}. The dataset consists of $150$ samples, divided equally into three classes, each sample having $4$ attributes. \tablename~\ref{tab:interlacing} shows the interlacing property of the largest three eigenvalues of each Gram matrix, $\blue \bK$ and $\red \bKc$, as settled in Theorem~\ref{th:ineq_lambda}. \figurename~\ref{fig:Theorem8} illustrates Theorem~\ref{th:lambdac}, where a lower bound on the largest $t$ eigenvalues of $\bKc$ is derived, while the specific cases of $t=1$ and $t=n$ are shown.

We illustrate next the impact of data centering in kernel-based methods. To this end, we consider a set of $n=200$ two-dimensional data generated from a banana-shaped distribution, with $(x_j,y_j) = (\zeta_i,\zeta_i^2+\xi)$ where $\zeta_i$ follows a uniform distribution on $[-1 ~ 1]$ and $\xi$ follows a zero-mean Gaussian distribution with a $0.2$ standard deviation. The Gram matrices were constructed by using the Gaussian kernel $\exp(-\tfrac{1}{2\sigma^2} \|\bx_i - \bx_j\|^2)$, where the bandwidth parameter was (naively) set to $\sigma = 0.5$. \tablename~\ref{tab:interlacing_banana} illustrates the interlacing property of the largest eigenvalues of the two kernel matrices, the non-centered $\bK$ with entries $\kappa(\bx_i,\bx_j)$ and the corresponding centered matrix $\bKc$. \figurename~\ref{fig:banana} shows the contours of the first five principal functions, when data are centered (implicitly) in the feature space (first row), and when data is not centered (second row). This illustrates that the first principal function of the non-centered case is related to the data mean, while the other principal functions are similar to those obtained from the centered case, with one order higher, namely results from $(\lambdac_i,\balphac_i)$ are comparable to results from $(\lambda_{i+1},\balpha_{i+1})$.


\begin{table}[t]
\caption{Illustration of the interlacing property of the eigenvalues of $\blue \bK$ and $\red \bKc$. In these experiments, the values are obtained from the {\em iris} dataset.}
\begin{center}
\renewcommand*{\arraystretch}{1.3}
\begin{tabular}{|c|c|c|c|c|c|c|c|c|c|c|c}
	\cline{2-2} \cline{4-4} \cline{6-6} \cline{8-8}
	\multicolumn{1}{c|}{} & $\blue\lambda_5$ & & $\blue\lambda_3$ & & $\blue\lambda_2$ & & $\blue\lambda_1$ \\
	\cline{1-1} \cline{3-3} \cline{5-5} \cline{7-7}
	\multicolumn{1}{|c@{~~$\leq$\!\!}}{\red$\!0.51\!$} & \multicolumn{1}{c@{~~$\leq$\!\!}}{\blue$\!0.67\!$} & \multicolumn{1}{c@{~~$\leq$\!\!}}{\red$\!2.39\!$} & \multicolumn{1}{c@{~~$\leq$\!\!}}{\blue$\!3.27\!$} & \multicolumn{1}{c@{~~$\leq$\!\!}}{\red$\!5.48\!$} & \multicolumn{1}{c@{~~$\leq$\!\!}}{\blue$\!59.72\!$} & \multicolumn{1}{c@{~~$\leq$\!\!}}{\red$\!98.45\!$} & \blue$\!101.68\!$ \\
	\cline{2-2} \cline{4-4} \cline{6-6} \cline{8-8}
	$\red\lambdac_5$ & & $\red\lambdac_3$ & & $\red\lambdac_2$ & & $\red\lambdac_1$ & \multicolumn{1}{|c}{} \\
	\cline{1-1} \cline{3-3} \cline{5-5} \cline{7-7}
\end{tabular}
\end{center}
\label{tab:interlacing}
\end{table}

\begin{figure}[t]
\begin{center}
\psfragscanon
\psfrag{t}[c][c]{\large$t$}
\psfrag{t1}[c][c]{\small~$1$}
\psfrag{t10}[c][c]{\small~$10$}
\psfrag{t100}[c][c]{\small~$100$}
\psfrag{t150}[c][c]{\small~$n$}
\hspace*{-1.5cm}%
\includegraphics[width=.63\textwidth]{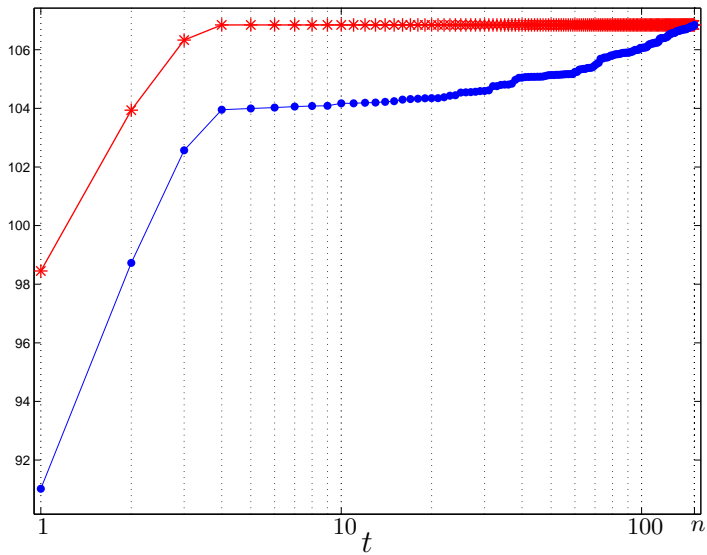}
\caption{Illustration of Theorem~\ref{th:lambdac}, in the case of the {\em iris} dataset with $n=150$. For any $t=1,2, \ldots, n$, the cumulative sum $\sum_{i=1}^t \lambdac_i$ (shown with~{\red $*$}) is greater than the cumulative sum of $d_i' = \lambda_i + (\|\bmu\|^2 - \tfrac{2}{n} \lambda_i) ~ (\balpha_i^\top \bone)^2$  (shown with~{\blue $\bullet$}), given in non-increasing order. We see that, for $t=1$, we have $\max_{i=1,\ldots,n} d_i' \leq \lambdac_1$, while we get the equality $\sum_{i=1}^t d_i' = \sum_{i=1}^t \lambdac_i$ when $t=n$.}
\label{fig:Theorem8}
\end{center}
\end{figure}

\begin{table}[t]
\caption{Illustration of the interlacing property of the eigenvalues of the kernel matrices, $\blue \bK$ and $\red \bKc$, where the Gaussian kernel is used. In these experiments, the values are obtained from the banana-shaped dataset.}
\begin{center}
\renewcommand*{\arraystretch}{1.3}
\begin{tabular}{|c|c|c|c|c|c|c|c|c|c|c|c}
	\cline{2-2} \cline{4-4} \cline{6-6} \cline{8-8}
	\multicolumn{1}{c|}{} & $\blue\lambda_5$ & & $\blue\lambda_3$ & & $\blue\lambda_2$ & & $\blue\lambda_1$ \\
	\cline{1-1} \cline{3-3} \cline{5-5} \cline{7-7}
	\multicolumn{1}{|c@{~~$\leq$\!\!}}{\red$\!10.17\!$} & \multicolumn{1}{c@{~~$\leq$\!\!}}{\blue$\!15.18\!$} & \multicolumn{1}{c@{~~$\leq$\!\!}}{\red$\!15.23\!$} & \multicolumn{1}{c@{~~$\leq$\!\!}}{\blue$\!26.73\!$} & \multicolumn{1}{c@{~~$\leq$\!\!}}{\red$\!31.33\!$} & \multicolumn{1}{c@{~~$\leq$\!\!}}{\blue$\!47.61\!$} &
		\multicolumn{1}{c@{~~$\leq$\!\!}}{\red$\!47.62\!$} & \blue$\!84.51\!$ \\
	\cline{2-2} \cline{4-4} \cline{6-6} \cline{8-8}
	$\red\lambdac_5$ & & $\red\lambdac_3$ & & $\red\lambdac_2$ & & $\red\lambdac_1$ & \multicolumn{1}{|c}{} \\
	\cline{1-1} \cline{3-3} \cline{5-5} \cline{7-7}
\end{tabular}
\end{center}
\bigskip\bigskip
\label{tab:interlacing_banana}
\end{table}

\begin{figure*}[t]
\hspace*{-2.5cm}%
\begin{overpic}[width=1.25\textwidth]{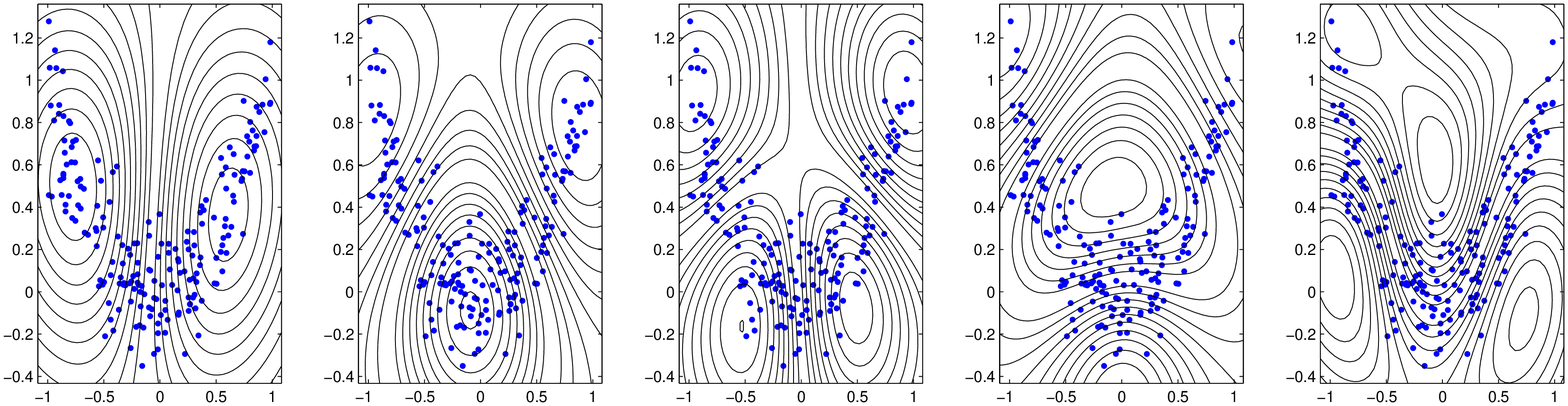}
	\put(9.5,8){\rotatebox{90}{centered}}
	\put(15.5,22.5){ $(\lambdac_1,\balphac_1)$}
	\put(31.75,22.5){ $(\lambdac_2,\balphac_2)$}
	\put(48.25,22.5){ $(\lambdac_3,\balphac_3)$}
	\put(64.5,22.5){ $(\lambdac_4,\balphac_4)$}
	\put(80.75,22.5){ $(\lambdac_5,\balphac_5)$}
\end{overpic}

\hspace*{-2.5cm}%
\begin{overpic}[width=1.25\textwidth]{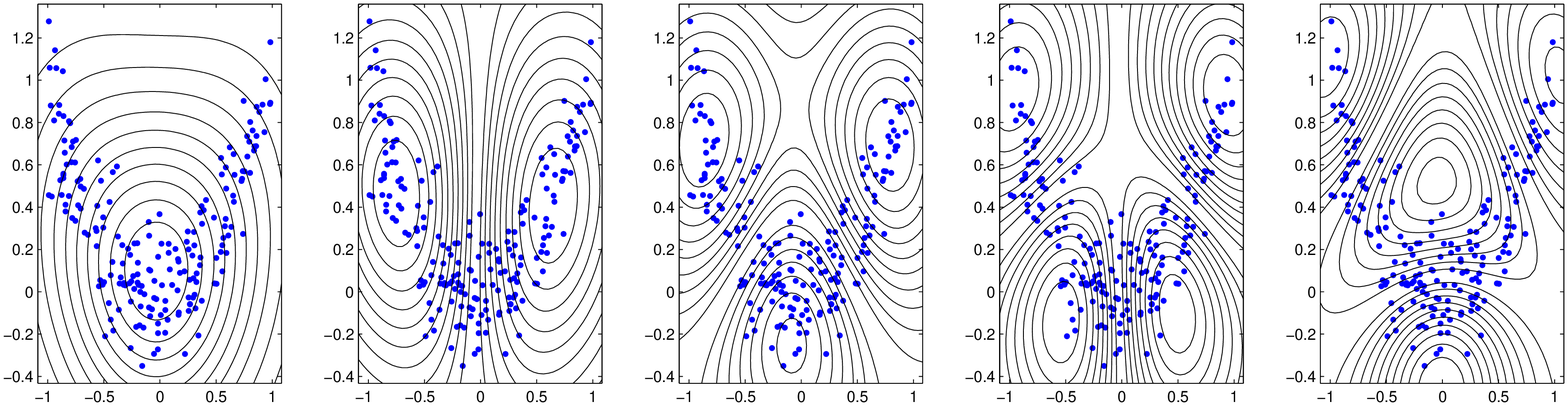}
	\put(9.5,8){\rotatebox{90}{non-centered}}
	\put(16.25,22.5){ $(\lambda_1,\balpha_1)$}
	\put(32.5,22.5){ $(\lambda_2,\balpha_2)$}
	\put(48.75,22.5){ $(\lambda_3,\balpha_3)$}
	\put(65,22.5){ $(\lambda_4,\balpha_4)$}
	\put(81.25,22.5){ $(\lambda_5,\balpha_5)$}
\put(25,22.75){\red\huge$\boldsymbol{\searrow}\!\!\!\!\!\!\!\boldsymbol{\nwarrow}$}
\put(41.5,22.75){\red\huge$\boldsymbol{\searrow}\!\!\!\!\!\!\!\boldsymbol{\nwarrow}$}
\put(57.75,22.75){\red\huge$\boldsymbol{\searrow}\!\!\!\!\!\!\!\boldsymbol{\nwarrow}$}
\put(74,22.75){\red\huge$\boldsymbol{\searrow}\!\!\!\!\!\!\!\boldsymbol{\nwarrow}$}
\end{overpic}
\caption{Illustration of the contours of the first five principal functions from the kernel-based PCA with the Gaussian kernel, with $n=200$ samples (given in blue dots $\blue \boldsymbol\cdot$) from a banana-shaped distribution. The first raw of figures is obtained with the eigendecomposition of the centered matrix $\bKc$, while the second row corresponds to the non-centered case with $\bK$.}
\label{fig:banana}
\end{figure*}

\section{Final remarks}

The main objective of this paper was to bridge the gap between centered and uncentered data. We studied the impact of centering data on the eigendecomposition of the Gram matrix, thereby of benefit to most kernel-based methods. To be more specific in this paper, we explored the eigendecomposition of the covariance matrix, with results that corroborate recent work on conventional PCA. Our key motivation was to reconcile the centered Gram matrix in PCA and the non-centered Gram matrix, such as with the ECA for nonparametric density estimation. Moreover, we provided several extensions of our main results, beyond the conventional centering issue.

Other techniques in manifold learning and dimensionality reduction can also take advantage of this work, include ISOMAP, locally-linear embedding, eigenmaps, and spectral clustering, to name a few. Further future work will address the issue of the impact of kernel functions in the centering issue, as well as the impact of centering the data in the input space, as opposed to the implicit centering in the feature space with kernel-based methods. We will also study connections with spectral analysis in random matrix theory \cite{RandomMatrix09}.

\bigskip\bigskip





\bibliography{biblio_ph,bibdesk_Paul}
\bibliographystyle{ieeetr}

\begin{biography}[{}]
{Paul Honeine} (M'07) was born in Beirut, Lebanon, on October 2, 1977. He received the Dipl.-Ing. degree in mechanical engineering in 2002 and the M.Sc. degree in industrial control in 2003, both from the Faculty of Engineering, the Lebanese University, Lebanon. In 2007, he received the Ph.D. degree in Systems Optimisation and Security from the University of Technology of Troyes, France, and was a Postdoctoral Research associate with the Systems Modeling and Dependability Laboratory, from 2007 to 2008. Since September 2008, he has been an assistant Professor at the University of Technology of Troyes, France. His research interests include nonstationary signal analysis and classification, nonlinear signal processing, sparse representations, machine learning, and wireless sensor networks. He is the co-author (with C. Richard) of the 2009 Best Paper Award at the IEEE Workshop on Machine Learning for Signal Processing.
\end{biography}

\end{document}